%% file: main.tex
\documentclass[conference]{IEEEtran}
\IEEEoverridecommandlockouts
\usepackage{cite}
\usepackage{amsmath,amssymb,amsfonts}
\usepackage{amsthm}
\usepackage{algorithmic}
\usepackage{graphicx}
\usepackage{textcomp}
\usepackage{xcolor}
\usepackage{algorithm}
\usepackage{hyperref}
\usepackage{tabu}
\usepackage{url}
\renewcommand{\L}{\mathcal{L}}
\newcommand{\st}{\hat{\theta}^{\star}}
\newcommand{\nt}{\hat{\theta}^{Nt}_{-z_i}}
\newcommand{\bt}{\hat{\theta}^{\star}_{-z_i}}

\newtheorem{thm}{Theorem}

\newtheorem{pro}[thm]{Proposition}

\def\BibTeX{{\rm B\kern-.05em{\sc i\kern-.025em b}\kern-.08em
    T\kern-.1667em\lower.7ex\hbox{E}\kern-.125emX}}
\begin{document}

\title{Model-Agnostic Explanations using \\ Minimal Forcing Subsets
}

\author{
\IEEEauthorblockN{Xing Han}
\IEEEauthorblockA{\textit{Department of Electrical and Computer Engineering} \\
\textit{University of Texas at Austin}\\
Austin, USA \\
\texttt{aaronhan223@utexas.edu}}
\and
\IEEEauthorblockN{Joydeep Ghosh}
\IEEEauthorblockA{\textit{Department of Electrical and Computer Engineering} \\
\textit{University of Texas at Austin}\\
Austin, USA \\
\texttt{jghosh@utexas.edu}}
}

\maketitle

\begin{abstract}
How can we find a subset of training samples that are most responsible for a specific prediction made by a complex black-box machine learning model? More generally, how can we explain the model's decisions to end-users in a transparent way? We propose a new model-agnostic algorithm to identify a minimal set of training samples that are indispensable for a given model's decision at a particular test point, i.e., the model's decision would have changed upon the removal of this subset from the training dataset. Our algorithm identifies such a set of ``indispensable'' samples iteratively by solving a constrained optimization problem. Further, we speed up the algorithm through efficient approximations and provide theoretical justification for its performance. To demonstrate the applicability and effectiveness of our approach, we apply it to a variety of tasks including data poisoning detection, training set debugging and understanding loan decisions. The results show that our algorithm is an effective and easy-to-comprehend tool that helps to better understand local model behavior, and therefore facilitates the adoption of machine learning in domains where such understanding is a requisite. 
\end{abstract}

\input{tex/1_intro.tex}
\input{tex/2_related.tex}
\input{tex/3_method.tex}
\input{tex/4_experiment.tex}

\input{tex/5_conclusion.tex}

\bibliography{reference}
\bibliographystyle{ieeetr}

\end{document}

%% file: tex/1_intro.tex
\section{Introduction}
Machine learning has achieved ground-breaking progress in various sub-fields, from computer vision to speech recognition. Traditionally, performance measures such as accuracy or AUC were emphasized during algorithm design; however, as such models are increasingly deployed in high-stakes situations, practitioners have realized that the lack of transparency of black-box models significantly limits their acceptance and deployment in automated decision making systems that directly affect humans in significant ways\cite{carvalho2019machine, doshi2017towards}. For example, deep learning has been widely used to make movie recommendations, as the impact of an occasional poor recommendation for a client is tolerable. In contrast, a wrong decision provided by a machine learning model for medical diagnosis, autonomous-driving or financial transactions, can lead to more severe outcomes: irreversible damage to health, fatal car accidents, or multi-million dollar loss of capital. In such high-stakes scenarios, it would not be possible (and actually not reasonable) to rely on machine learning algorithms to make critical decisions without developing a solid understanding of the model behavior and communicating these insights with the end-users in a transparent way.


Motivated by the growing interest in model interpretation, a variety of techniques, such as influence functions \cite{koh2017understanding}, prototypes \cite{bien2011prototype}, and representer points \cite{yeh2018representer}, have been proposed to provide insights into black-box machine learning models. However, most of the existing explanation techniques are quite complicated and meant for a data science audience. There is a great gap in developing explanation methods that are able to communicate the model decisions clearly and intuitively to a broader community \cite{bhatt2020explainable}. Such transparent user-oriented model interpretations are crucial to build trust in machine learning systems and facilitate their applications in more real-world scenarios. This paper aims to help fill this gap by introducing an effective and easy-to-comprehend algorithm to cast light on local model predictions, especially erroneous ones.

We frame our problem setting as follows: assume $ \mathcal A(D, x^*) \in \{0, 1\}$ is a binary decision made on data instance $x^*$ by algorithm $\mathcal A$ based on training data $D$. We want to understand which part of the training data $D$ was most influential for a given decision. We say that a subset of training sample(s), denoted by $\mathcal{S}$, is a forcing set for the decision $a(x^*) = \mathcal A(D, x^*)$, 
if the decision would be different if this set was not present in the training data, i.e.,  $\mathcal A(D, x^*)  \neq  \mathcal A(D\setminus \mathcal{S}, x^*)$, where $D \setminus \mathcal{S}$ is the modified data-set by removing training sample(s) $\mathcal{S}$. Furthermore, we aim to identify a minimal subset $\mathcal{S}$ in an efficient way. We then show that such a set not only
provides a "prototype-based" explanation of model prediction, but can also be used in a variety of associated diagnostics that help make both model behavior and break-points more transparent.

Key aspects of our interpretation framework include:
\begin{enumerate}
    \item introduction of a sample-based local explanation method that identifies a minimal set of decision-related training samples in a ``knockout'' fashion
    \item development of an iterative algorithm and associated approximation methods to efficiently find such a set of training samples
    \item illustration of the benefits of our method in various settings, showing superior performance both quantitatively and qualitatively.
\end{enumerate}

%% file: tex/2_related.tex
\section{Related Work}
There is a rapidly expanding literature on model explainability in recent years, including both global interpretation approaches \cite{lakkaraju2016interpretable, ustun2016supersparse, wang2015falling} and local explanations \cite{baehrens2010explain, ribeiro2016should, kononenko2010efficient, ribeiro2018anchors, lundberg2017unified}. A global approach attempts to learn models that are more understandable due to the nature of their functional form rather than focussing on explaining specific decisions. For example, interpretable decision sets by \cite{lakkaraju2016interpretable} provides a joint framework for description and prediction, where the model is trained to maximize a multi-objective function, and produce succinct and explainable decision rules. Another approach is to learn a simpler, more understandable surrogate model that tries to mimic the behavior of a complex model.

On the other hand, local approaches aim to shed light on the model decisions at specific points. Two notable approaches to local explanations are prototype-based explanations and feature-based explanations. The former, which is related to case-based or exemplar-based reasoning, focus on identifying one or a small subset of samples that were most
influential for the specific decision being considered. Our algorithm belongs to this category. Prior work has developed other techniques for sample-based explanations. For example, the prototype selection method \cite{bien2011prototype} provides a set of ``representative'' samples chosen from the data set. The approach that influences our work the most is the use of influence functions to find a single most critical instance for a test point prediction \cite{koh2017understanding}. Essentially, this method computes sample influence based on perturbations of a test data point and has been subsequently expanded to determine a set of the most significant points of a pre-specified size \cite{khanna2018interpreting}. However, there is no guarantee that removing this set would have changed the outcome. Similarly, \cite{yeh2018representer} calculates representer values to measure the importance of each training point by decomposing the pre-activation prediction values based on a representer theorem. This method is limited to neural network models trained in a particular way. We will show that our method outperforms these two and other recent state-of-the-art methods in multiple domains.

In contrast, feature-based or feature attribution methods attempt to find the most important features determining the model's prediction. For deep neural network models, a typical feature-attribution-based method is saliency maps \cite{simonyan2013deep}, which is based on computing the gradient of the class score with respect to the input image for image classification problems. Another example is Layer-wise Relevance Propagation \cite{binder2016layer}, which decomposes the prediction of a deep neural network computed over a sample down to relevance scores for the single input dimensions of the sample. Furthermore, DeepLIFT \cite{shrikumar2017learning} improves layer-wise propagation by solving the saturation problem. 
Approaches such as SHAP are inspired by the Shapley value \cite{shapley1953value} from cooperative game theory, and are applicable for feature attribution for general tabular data, and just not tied to image understanding \cite{lundberg2017unified, chen2018shapley, strumbelj2010efficient, lundberg2020local, mase2019explaining, ghorbani2019data, ghorbani2020neuron}.
These methods have several nice properties but need to compute models over all subsets of features. Thus, in general, they suffer from computational problems, so providing an efficient solution is an active area of research \cite{kwon2020efficient}.

Our approach has the flavor of counterfactuals: in an alternate world where the set $\mathcal{S}$ was removed from the training set, the outcome would have been different. The notion of ``Counterfactual explanations" in the form of what change in an input point would have flipped the outcome was proposed in \cite{wachter2017counterfactual}, as a way to point to ``recourse". Such a view has also been used in various data types like document \cite{martens2014explaining}, image \cite{goyal2019counterfactual} and time series classifications \cite{karlsson2018explainable}. \cite{binns2018s, dodge2019explaining} have demonstrated that users prefer counterfactual explanations over case-based reasoning, while \cite{fernandez2020explaining} has also given examples to show that counterfactual explanations are superior to feature importance methods. Our approach has aspects of both sample-based and counterfactual explanations to provide more intuitive insights into what most affected a model's decision.


%% file: tex/3_method.tex
\section{Transparent Interpretation Framework}\label{sec:method}


\begin{figure*}[t]
\begin{minipage}{\textwidth}
\centering
\begin{tabular}{@{\hspace{-1.4ex}} c @{\hspace{-1.4ex}} @{\hspace{-1.4ex}} c @{\hspace{-1.4ex}} @{\hspace{-1.4ex}} c @{\hspace{-1.4ex}} @{\hspace{-1.4ex}} c @{\hspace{-1.4ex}}}
    \begin{tabular}{c}
    \includegraphics[width=.25\textwidth]{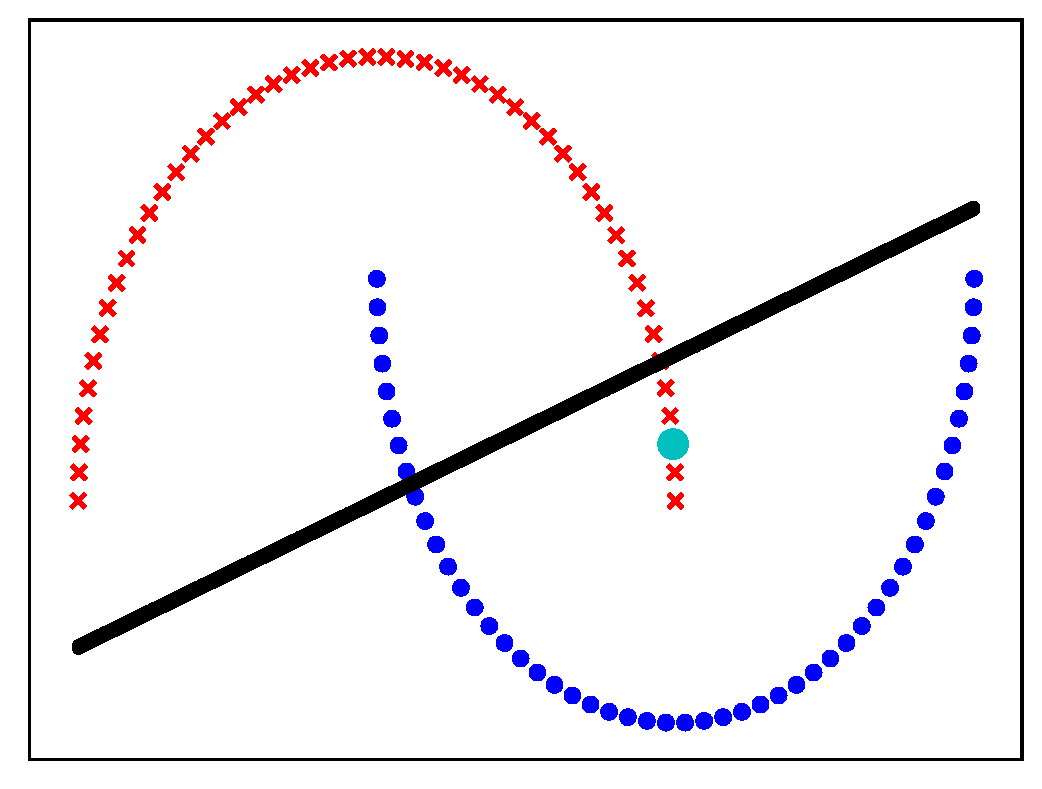}
    \\
    {\small{(a)}}
    \end{tabular} &
    \begin{tabular}{c}
    \includegraphics[width=.25\textwidth]{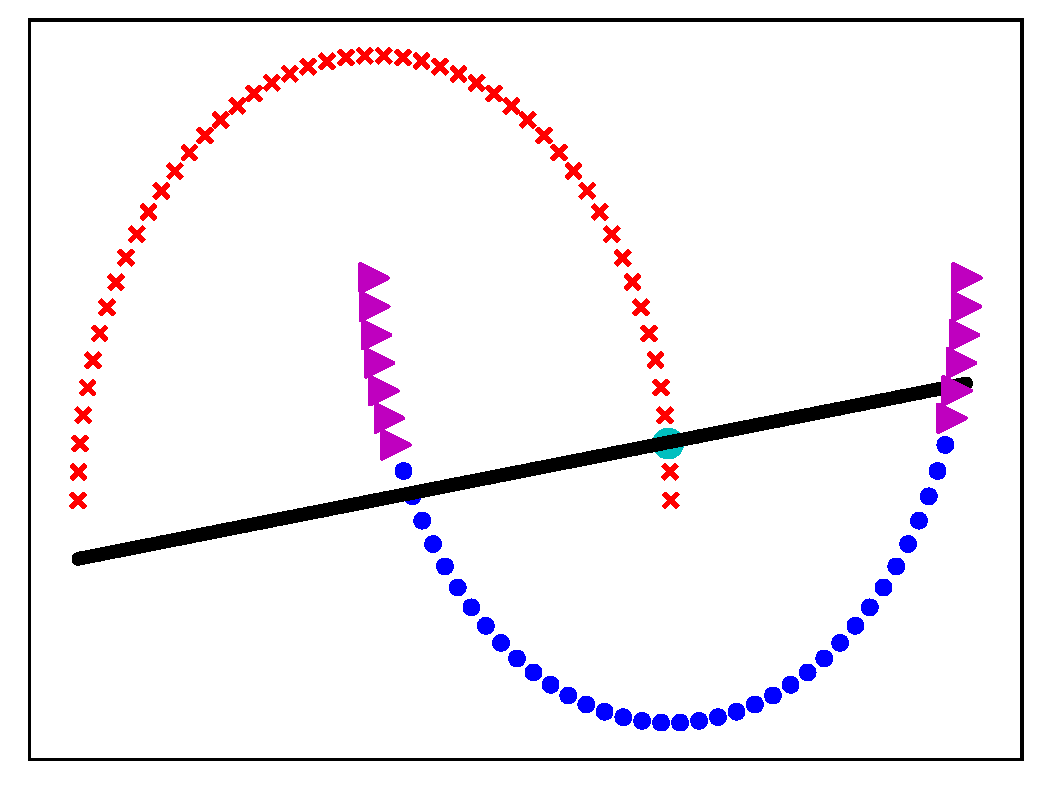}
    \\
    {\small{(b)}}
    \end{tabular} & 
    \begin{tabular}{c}
    \includegraphics[width=.25\textwidth]{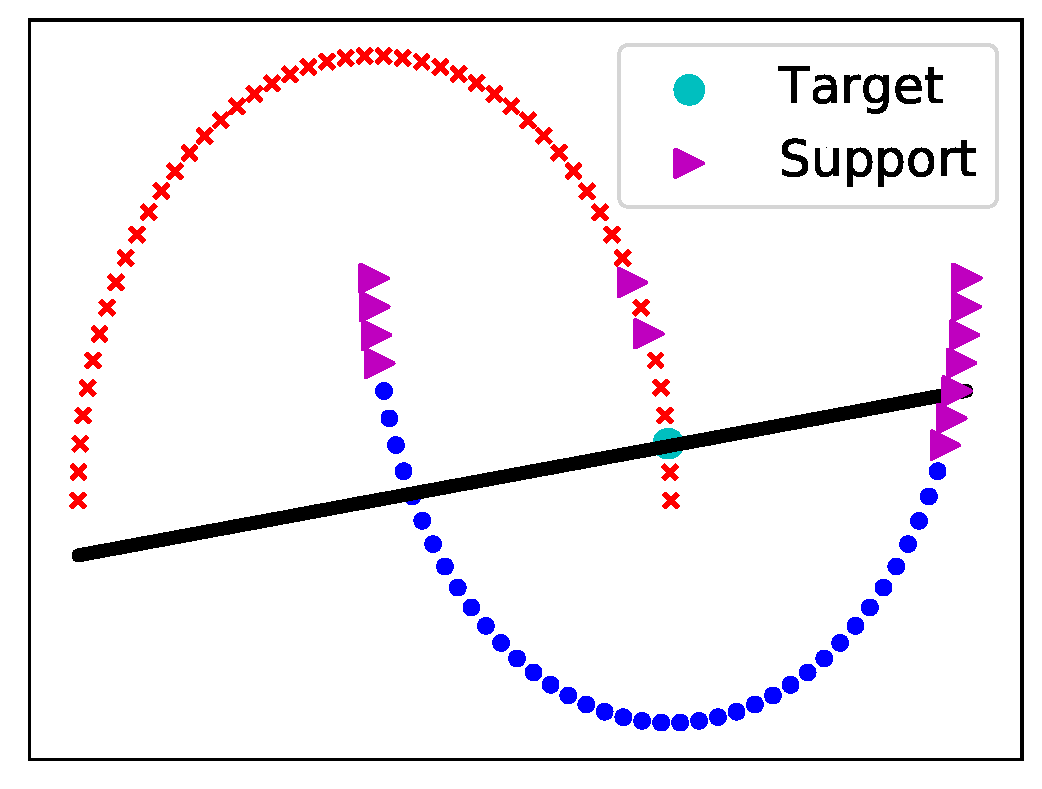} 
    \\
    {\small{(c)}}
    \end{tabular} &
    \begin{tabular}{c}
    \includegraphics[width=.25\textwidth]{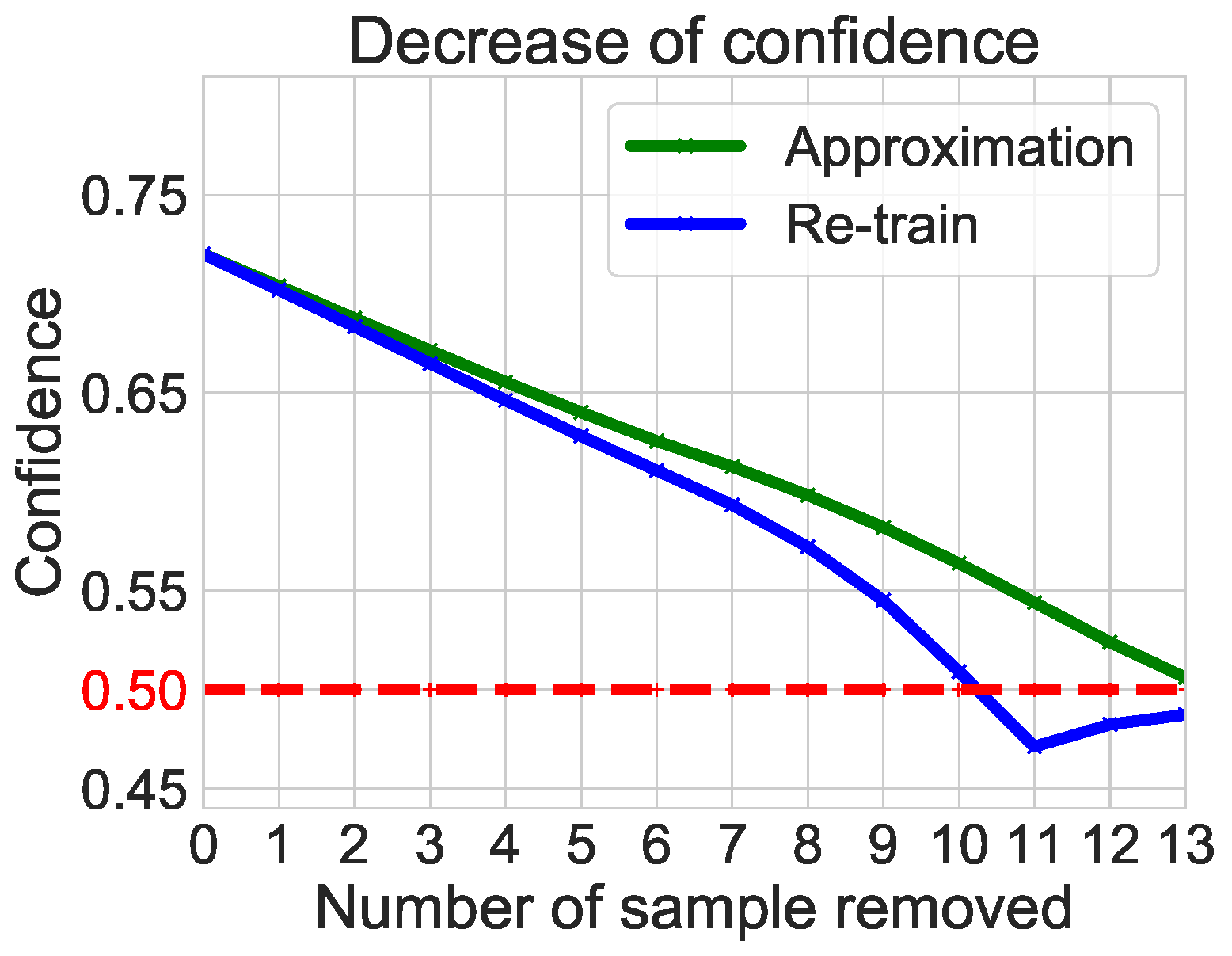}
    \\
    {\small{(d)}}
    \end{tabular}
    \end{tabular}
\end{minipage}
\caption[Caption for LOF]{Behavior of classifier on half-moon dataset. (a) Initial decision boundary. 
(b) Decision boundary after removing supports by approximation. 
(c) Decision boundary after removing supports by re-training.
(d) Decrease of confidence as supports removed\footnotemark.}
\label{fig:toy}
\end{figure*}



\paragraph{Problem Setup}
Consider a classification problem where the decision $a \in \{0, 1\}$ is binary,
and we can obtain an optimal classifier $\hat{\pi}$ using some machine learning algorithm $\mathcal{A}$ applied on a training dataset $\mathcal{D} = \{z_i:= (x_i, y_i)\}_{i=1}^n$ based on some objective loss:
\begin{align}
    \hat{\pi} \leftarrow \arg\min_{\theta} L(\theta) =\arg\min_{\theta} \frac{1}{n}\sum_{i = 1}^n l(z_i, \theta),\label{equ:unconstrain_loss}
\end{align}
where $l$ is the sample loss (e.g., cross entropy) of data instance $z_i$.
The goal of our method is to explain why the classifier $\hat{\pi}$ gives a prediction $a(x^*)$ instead of  $\bar{a}(x^*) = 1 - a(x^*)$ at a data point $x^*$.
We formulate the problem as finding a set of data instances which contributes the most to the claim $C$:
\begin{align*}
    p(c_0 ~|~ x^*, \hat{\pi}) > p(c_1~|~ x^*, \hat{\pi}),
\end{align*}
where $p(c_0~|~ x^*, \hat{\pi})$ is the prediction probability of decision $a(x^*)$ being class $0$ given by the classifier.


\paragraph{Explanation using Minimal Forcing Subset (MFS)}
To find the set of data points that can explain the claim $C$ efficiently,
we first consider the following counterfactual, which enforces the claim $C$ to be explained as the opposite:
\begin{align}
    \hat{\pi}_{c} \leftarrow \arg\min_{\theta}L(\theta), s.t. ~p(c_0| x^*,\hat{\pi}_{c}) < p(c_1 | x^*,\hat{\pi}_{c}) - \epsilon, \label{equ:constrain_obj}
\end{align}
where $\epsilon$ is a slack variable.
Solving \eqref{equ:constrain_obj} gives another solution that attempts to minimize the loss function but with the opposite decision at the data point $x^*$.
Due to the additional constraint, $L(\hat{\pi}_{c})$ must be greater than (or equal to) $L(\hat{\pi})$.
Recognizing that the data instance $z_{\hat{i}}$ contributing the most to the claim $C$ tends to be the one which violates constraints the most in the problem (\ref{equ:constrain_obj}), 
we can select the most responsible instance for the claim $C$ as 
\begin{equation}
    \hat{i} =\arg\max_{i}\left\{\max [l_{\epsilon, x^*}(z_i, \hat{\pi}_c) - l(z_i, \hat{\pi}), 0]\right\}.
    \label{eq:select_instance}
\end{equation}
Then, we remove $z_{\hat{i}}$ from the training dataset $\mathcal{D}$ and add $z_{\hat{i}}$ to the MFS $\mathcal{S}$. We subsequently solve \eqref{equ:unconstrain_loss} and \eqref{equ:constrain_obj} again but this time on the updated dataset $\mathcal{D}/{z_i}$, and select the second most responsible instance in the same way as in the previous round. We repeat this process until the difference between $L(\hat{\pi}_c)$ and $L({\hat{\pi}})$ is statistically insignificant.
This yields a set of data instances  $\mathcal{S}$ that is most responsible for the given statement. 

It is worth highlighting two features of our approach. First, the principle guiding the selection of the most responsible instance is intuitive and easy for users to grasp. 
Second, the exiting condition of the iteration process ensures the efficiency of the resulting $\mathcal{S}$. By efficiency, we mean the elements in $\mathcal{S}$ are quasi-minimized. Specifically, if $\mathcal{S}$ satisfies the explanation condition we defined above, then it is obvious that $\mathcal{S}$ plus some extra samples $e$ together also qualifies as an explanation. However, the union of $\mathcal{S}$ and $e$ is less efficient than $\mathcal{S}$ itself, as the former contains redundant elements. 

\begin{algorithm}[t]
\small
\caption{Iteratively Construct MFS}
\label{alg:explain}
\begin{algorithmic}[1]
\STATE Given dataset $\mathcal{D} = \{z_i := (x_i, y_i)\}_{i=1}^{n}$,  the test point $x^*$, MFS $\mathcal{S} = \emptyset$, loss function $\L$, its approximation $f$ and parameter $\theta$. 
\STATE Initialize $\mathcal{D}^\prime = \mathcal{D}$, original classifier $\hat{\pi}_{\mathcal{D}}$.
\REPEAT
\STATE Obtain constrained classifier $\hat{\pi}_{\mathcal{D}/\mathcal{S}, \epsilon, x^*}$ by solving\\ $\hat{\theta}^{\prime} = \arg\underset{\theta}{\min} \{f(\theta) |  
     \log p (c_1 | x^*) - \log p (c_0 | x^*) - \epsilon \geq 0\}$; 
\STATE Select data instance by \\
$\hat{i} =\arg\underset{i}{\max}\left\{\max [l_{\epsilon, x^*}(z_i, \hat{\pi}_{\mathcal{D}/\mathcal{S}, \epsilon, x^*}) - l(z_i, \hat{\pi}_{\mathcal{D}/\mathcal{S}}), 0]\right\}$; 
\STATE Add data instance $z_{\hat{i}}$ to set $\mathcal{S}$;
\STATE Update original classifier $\hat{\pi}_{\mathcal{D}/\mathcal{S}}$ using one-step Newton \\ $\hat{\theta}^{Nt}_{-z_{\hat{i}}} \leftarrow \hat{\theta} + \frac{1}{n} \mathcal{H}_{\hat{\theta}}^{-1} \nabla_{\theta} \L(z_{\hat{i}}, \hat{\theta})$; 
\STATE Remove data instance $z_{\hat{i}}$ from $\mathcal{D}^\prime$: $\mathcal{D}' := \mathcal{D}' / z_{\hat{i}}$; 
\UNTIL the difference between {$\L(\hat{\pi}_{\mathcal{D}/\mathcal{S}, \epsilon, x^*})$ and $\L(\hat{\pi}_{\mathcal{D}/\mathcal{S}})$ is statistically insignificant (see section A part b).}
\RETURN MFS $\mathcal{S}$
\end{algorithmic}
\end{algorithm}

\footnotetext{Gifs can be found at \href{tiny.cc/67j8bz}{tiny.cc/67j8bz}.}

\subsection{Efficiently Construct MFS}

\paragraph{Approximation Methods}
We proposed an iterative algorithm (\ref{alg:explain}) to construct the MFS. A few critical problems are how to efficiently solve the constrained optimization problem (\ref{equ:constrain_obj}), as well as to avoid re-training the model each step as we construct the MFS. To speed up the constrained optimization, we first use a second-order Taylor expansion of the loss function $\L(\theta)$ parameterized by $\theta$ at $\hat{\theta}$: 
\begin{equation}
    f(\theta) = \L(\hat{\theta}) + (\theta - \hat{\theta})^{\top} \nabla \L(\hat{\theta}) + \frac{1}{2} (\theta - \hat{\theta})^{\top} \mathcal{H}_{\hat{\theta}} (\theta - \hat{\theta}),
    \label{eq:quadratic}
\end{equation}
where $\mathcal{H}_{\hat{\theta}}$ denotes the Hessian matrix of the loss function at $\hat{\theta}$. We then formulate a quadratic programming problem subject to linear inequality constraint:

\begin{align}
    \hat{\theta}^{\prime} = \arg\underset{\theta}{\min} \{f(\theta) |  
     \log p (c_1 | x^*, \theta) - \log p (c_0 | x^*, \theta) - \epsilon \geq 0\}.  \nonumber
\end{align}
Parameter $\hat{\theta}^{\prime}$ under the constraint can be obtained by solving the above problem via the Karush–Kuhn–Tucker (KKT) conditions \cite{boyd2004convex}. We use $\hat{\theta}^{\prime}$ to update the original parameter $\hat{\theta}$ until a sufficient number of iterations is reached. 

As we proceed with algorithm (\ref{alg:explain}), we need to obtain a series of classifiers $\hat{\pi}_{\mathcal{D} / \mathcal{S}}$ while removing instances from $\mathcal{D}$. To avoid re-training at each iteration, after a data instance $z_i$ has been selected and removed from the training set, we approximate its effect on the model parameters using a one-step Newton update: 
\begin{equation}
    \hat{\theta}^{Nt}_{-z_i} \leftarrow \hat{\theta} + \frac{1}{n} \mathcal{H}_{\hat{\theta}}^{-1} \nabla_{\theta} \L(z_i, \hat{\theta}).
    \label{eq:one_step_newton}
\end{equation}
This update can be facilitated by borrowing the result from the quadratic formula (\ref{eq:quadratic}) to avoid the calculation of the Hessian-vector products, which is equivalent to $(\theta - \hat{\theta})$. It is applied in step 7 of algorithm (\ref{alg:explain}), but this could also be substituted with re-trained classifier on $\mathcal{D} / \mathcal{S}$. We will discuss the difference between these two choices, but the default method in the following sections is the one-step Newton method.

\paragraph{Exiting Conditions} \label{sec:conditions}
We evaluate the effect of removing the identified set of responsible samples by comparing $\L(\hat{\pi}_c)$ and $\L(\hat{\pi})$. The loss under the counterfactual optimization is inherently greater than the original loss due to the imposed constraint. However, their difference may not be statistically distinguishable when sufficient points are removed. 
We use a stopping criteria of $\L(\hat{\pi}_c) + \delta < \L(\hat{\pi})$, where $\delta$ is arbitrary small constant. With a small threshold $\delta$, this stopping condition guarantees that the claim is falsified when the selected points are removed. 


\paragraph{Illustration Using Simulated Data}
 We demonstrate the working mechanism of our algorithm in Figure \ref{fig:toy} by applying it to the half-moon dataset, which contains 100 samples divided into two classes. We first train a linear classifier and select a target data point that is wrongly classified with confidence 74\% (Figure \ref{fig:toy}(a)). MFS of the target is then iteratively constructed following algorithm (\ref{alg:explain}). As a comparison, we use both one-step approximation and re-training (Figure \ref{fig:toy}(b) and Figure \ref{fig:toy}(c)). We find that the approximation method identifies 13 samples to support its decision. In comparison, the re-training method selects 11 samples (after removing MFS, the target becomes an ambiguous point that lies on the decision boundary). Figure \ref{fig:toy}(d) compares the change of confidence as each support is removed. Regardless of computation time, the re-training approach reaches the boundary faster, and the $12^{th}$ and $13^{th}$ samples are selected to balance the boundary.

\subsection{Error Bound for One-step Newton}
We now provide theoretical analysis to bound the error of the one-step Newton approximation within finite number of iterations of algorithm (\ref{alg:explain}). Result shows that the approximation error $E_{Nt}$ decays at a rate of $O(1/(\lambda_{min} + \alpha)^3),$ where $\alpha$ is the parameter that controls regularization strength (not shown in Eq. (\ref{equ:unconstrain_loss})), and $\lambda_{min}$ is the smallest eigenvalue of $\mathcal{H}_{\hat{\theta}}$. 

\begin{pro} (\textbf{Error Bound}) \label{pro:error_bound}
Assume the model $F$ is well-optimized on dataset $\mathcal{D} = \{z_i:= (x_i, y_i)\}_{i=1}^n$ with strongly convex loss function $\L$, $F(\theta)$ is $L_F$-Lipschitz and the Hessian $\mathcal{H}_{\hat{\theta}}$ is $L_{\mathcal{H}}$-Lipschitz. Denote the parameter set that is trained with and without $z_i$ as $\st$ and $\bt$, then

\begin{equation}
    E_{Nt} = \left|F(\bt) - F(\nt)\right| \leq \frac{n L_F L_{\mathcal{H}} N_g^2}{(\lambda_{min} + \alpha)^3},
\end{equation}
where $N_g$ is defined to be the largest gradient norm of the training samples at $\st$, i.e., $N_g = \underset{1\leq i \leq n}{\max} \|\nabla_{\theta} \L(z_i, \st)\|_2$.
\end{pro}

\begin{proof}
The proof technique is adapted to our setting from the analysis of Newton's method in convex optimization \cite{boyd2004convex} and influence function \cite{koh2019accuracy}.

We can bound $E_{Nt} = \left|F(\bt) - F(\nt)\right|$ through bounding the norm of the parameter difference $\left\|\bt - \nt\right\|_2,$ then by Lipschitz continuity that the gradient of $F$ is bounded by $L_F$, the bound on $F$ can be obtained accordingly.

Since $\L$ is strongly convex and empirical risk $\L_{-z_i}$ is minimized by $\bt$. Through the property of convex function, we can bound the parameter difference $\left\|\bt - \nt\right\|_2$ with respect to the gradient norm at $\nt$:

\begin{equation}
    \left\|\bt - \nt\right\|_2 \leq \frac{2}{\lambda_{min} + \alpha} \left\|\nabla_{\theta} \L_{-z_i}(\nt)\right\|_2 \nonumber
\end{equation}
We can then transform the problem into bounding $\left\|\nabla_{\theta} \L_{-z_i}(\nt)\right\|_2$.

Define the Newton step $\Delta \theta^{Nt}$ as $\nt - \st$, we first show that $\Delta \theta^{Nt}$ can be written in terms of the first and second order derivatives of the empirical risk $\L_{-z_i}$. According to Eq.(\ref{eq:one_step_newton}):

\begin{align}
    \nabla_{\theta}\L(z_i, \st) &= -\sum_{j=1, j\neq i}^n \nabla_{\theta} l(z_j, \st) = -\nabla_{\theta}\L_{-z_i}(\st). 
    \label{eq:first_order}
\end{align}
\begin{align}
    \mathcal{H}_{\hat{\theta}} &= \sum_{j=1, j \neq i}^n \nabla_{\theta}^2 l(z_j, \st) = \nabla_{\theta}^2 \L_{-z_i}(\st). 
    \label{eq:sec_order}
\end{align}
Note that the second equality of Eq.(\ref{eq:first_order}) is due to the fact that the gradient of a strongly convex function at the optimum is zero, i.e., $\sum_{i=1}^n \nabla_{\theta} l(z_i, \st) = 0$. Then we have
\begin{align}
    \Delta \theta^{Nt} = \mathcal{H}_{\hat{\theta}}^{-1} \nabla_{\theta}\L(z_i, \st) = -\left[\nabla_{\theta}^2 \L_{-z_i}(\st)\right]^{-1} \nabla_{\theta} \L_{-z_i}(\st). \nonumber
\end{align}
Combining the above results, we can then bound the gradient norm $\left\|\nabla_{\theta} \L_{-z_i}(\nt)\right\|_2$ as the following:
\begin{align*}
    & \left\|\nabla_{\theta} \L_{-z_i}(\nt)\right\|_2 \\
    &= \left\|\nabla_{\theta} \L_{-z_i}(\st + \Delta \theta^{Nt})\right\|_2 \\
    &= \left\|\nabla_{\theta} \L_{-z_i}(\st + \Delta \theta^{Nt}) - \nabla_{\theta} \L_{-z_i}(\st) - \nabla_{\theta}^2 \L_{-z_i}(\st) \Delta \theta^{Nt}\right\|_2 \\
    &= \left\| \int_0^1 \left[ \nabla_{\theta}^2 \L_{-z_i} (\st + t\Delta\theta^{Nt}) - \nabla_{\theta}^2 \L_{-z_i} (\st)\right] \Delta \theta^{Nt} dt \right\|_2 \\
    &\leq \frac{n L_{\mathcal{H}}}{2} \|\Delta \theta^{Nt}\|_2^2 \\
    &= \frac{n L_{\mathcal{H}}}{2} \left\| \left[\nabla_{\theta}^2 \L_{-z_i}(\st)\right]^{-1} \nabla_{\theta} \L_{-z_i}(\st) \right\|_2^2 \\
    &\leq \frac{n L_{\mathcal{H}}}{2(\lambda_{min} + \alpha)^2} \left\|\nabla_{\theta}^2 \L_{-z_i}(\st)\right\|_2^2 \\
    &\leq \frac{n L_{\mathcal{H}} N_g^2}{2(\lambda_{min} + \alpha)^2}.
\end{align*}
Putting this bound together with the $L_F-$Lipschitz of $F(\theta)$ we can get the result of Proposition 1.\\
\end{proof}
\textbf{Remark} ~Note that Proposition 1 provides the approximation error of one-step Newton in each iteration of the algorithm (\ref{alg:explain}). If we choose a training subset that has $m$ samples, the overall approximation error bound is then $\frac{n m^2 L_F L_{\mathcal{H}} N_g^2}{(\lambda_{min} + \alpha)^3}$. The approximation is very accurate when the parameter $\alpha$ is large, or the gradient norm $N_g$ and sample size $m$ is small. 

\begin{figure*}[t!]
\centering
{
\setlength{\tabcolsep}{1pt}
\renewcommand{\arraystretch}{1} 
\begin{tabu}{cc}
\includegraphics[width=.35\textwidth]{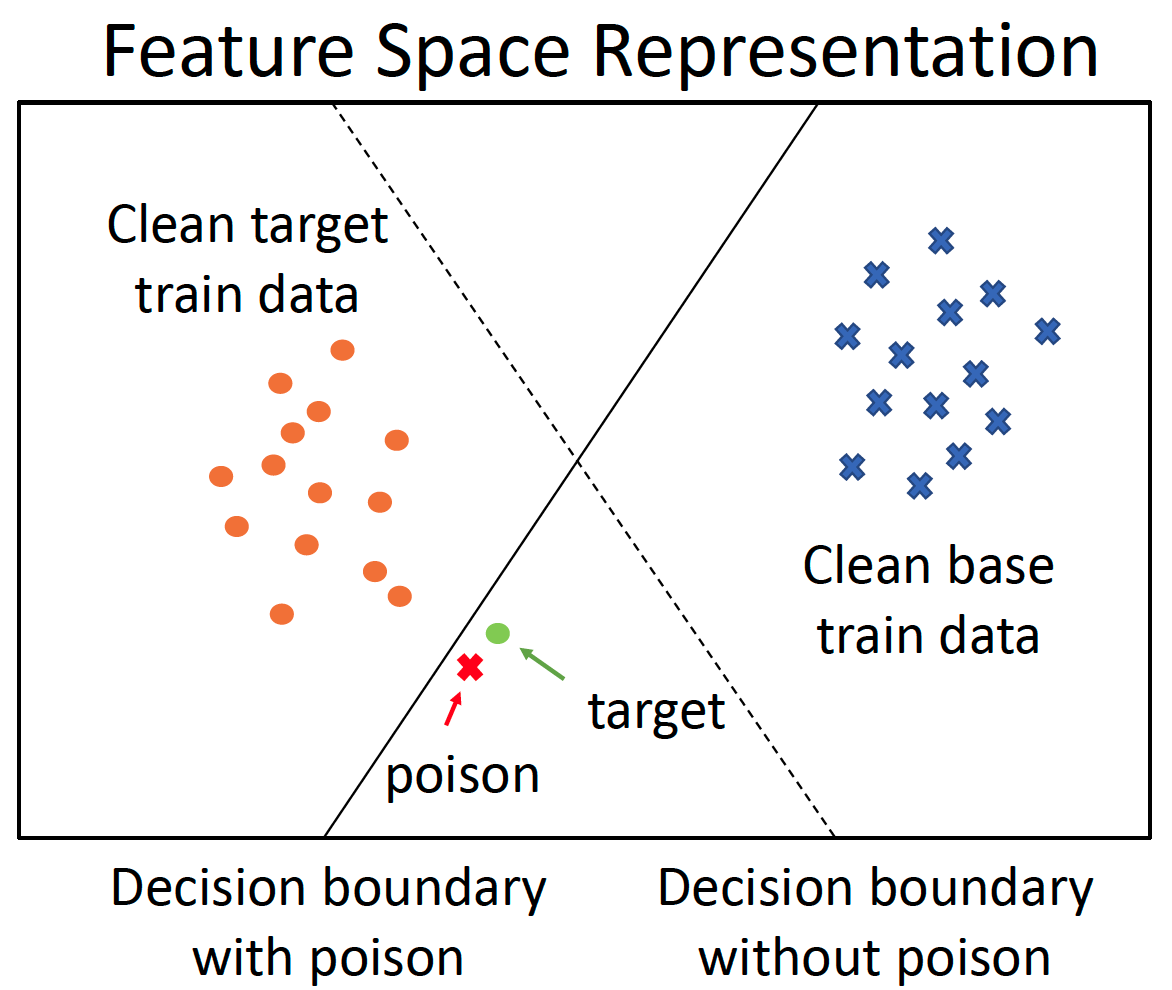} &
\includegraphics[width=.45\textwidth]{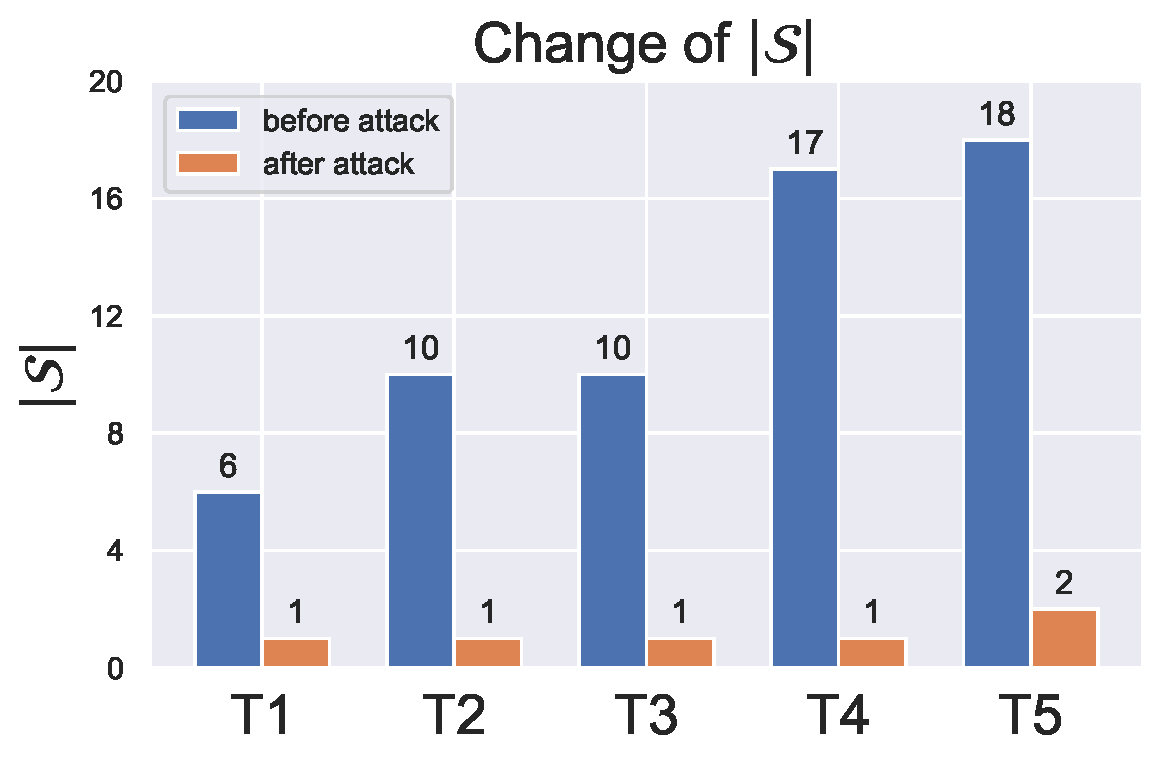} \\
(a) Feature Space Demonstration & (b) Data Poisoning Detection \\
\end{tabu}}
\caption{(a) An illustration of  how a successful attack might work by shifting the decision boundary. 
(b) Comparison of the MFS size $|\mathcal{S}|$ constructed by our algorithm, before and after attack. $|\mathcal{S}|$ decreases clearly if the training set contains poisoning attack.}
\label{fig:poison}
\end{figure*}

\begin{figure*}[t!]
\centering
{\includegraphics[width=\textwidth]{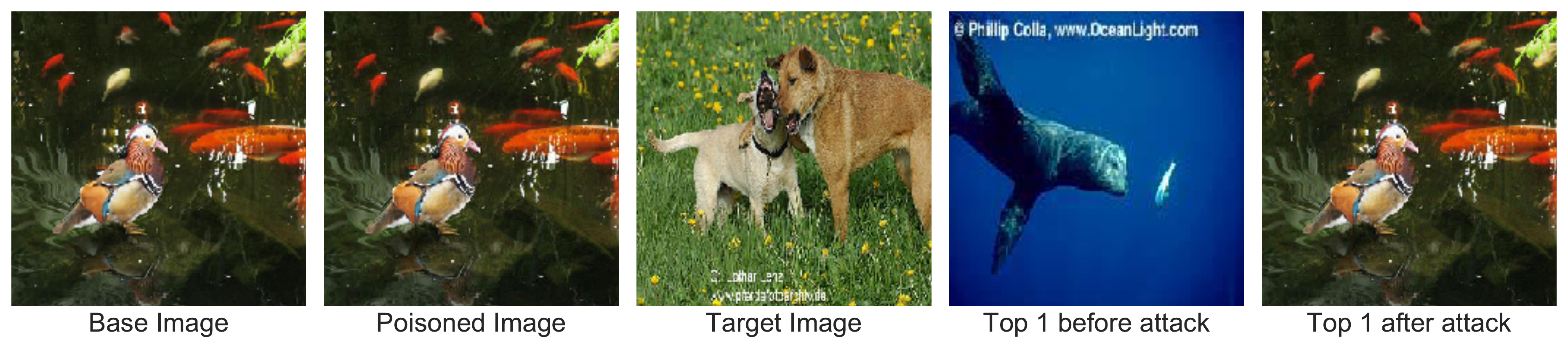}}
\caption{An example of data poisoning detection. From left to right: an image in the training set; its corresponding poisoned image crafted by \cite{shafahi2018poison}; a test set image we want to attack; 
the top influential image before and after attack.}
\label{fig:poison_imgs}
\end{figure*}

\begin{figure*}[ht]
\begin{minipage}{\textwidth}
    \centering
\begin{tabular}{@{\hspace{-1.4ex}} c @{\hspace{-1.4ex}} @{\hspace{-1.4ex}} c @{\hspace{-1.4ex}} @{\hspace{-1.4ex}} c @{\hspace{-1.4ex}}}
\begin{tabular}{c}
    \includegraphics[width=.33\textwidth]{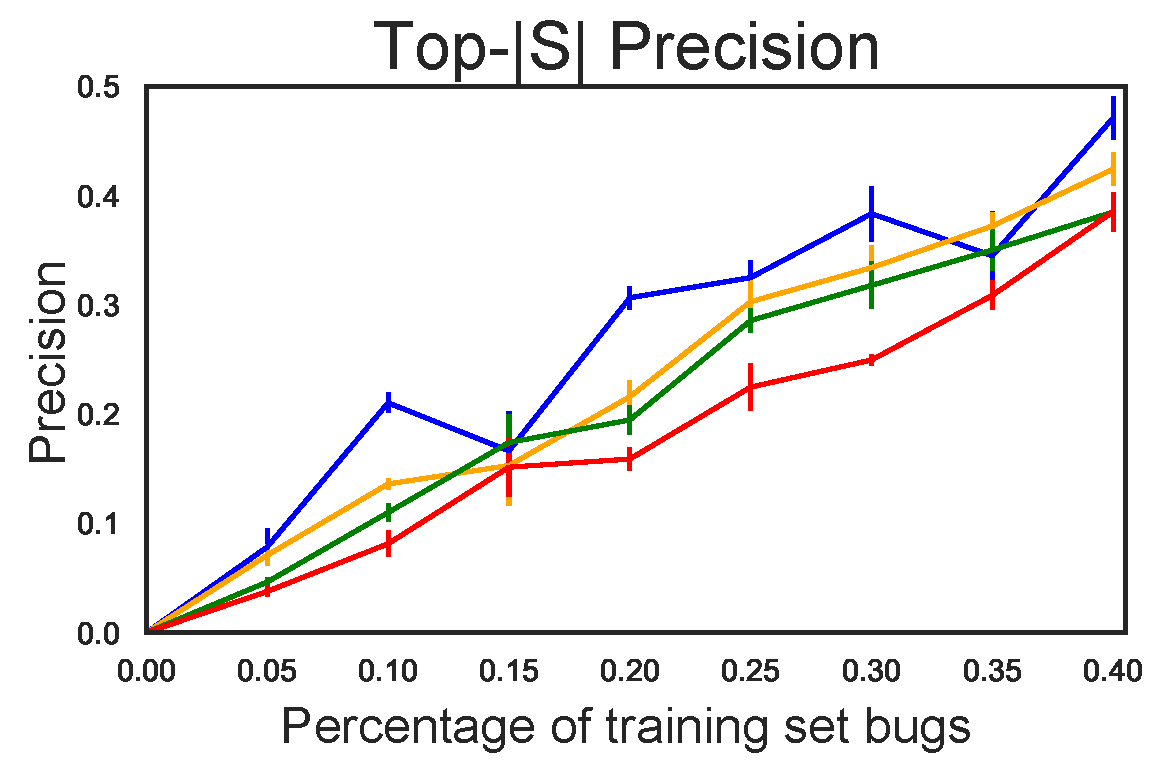}
    \\
    {\small{(a)}}
    \end{tabular} &
    \begin{tabular}{c}
    \includegraphics[width=.33\textwidth]{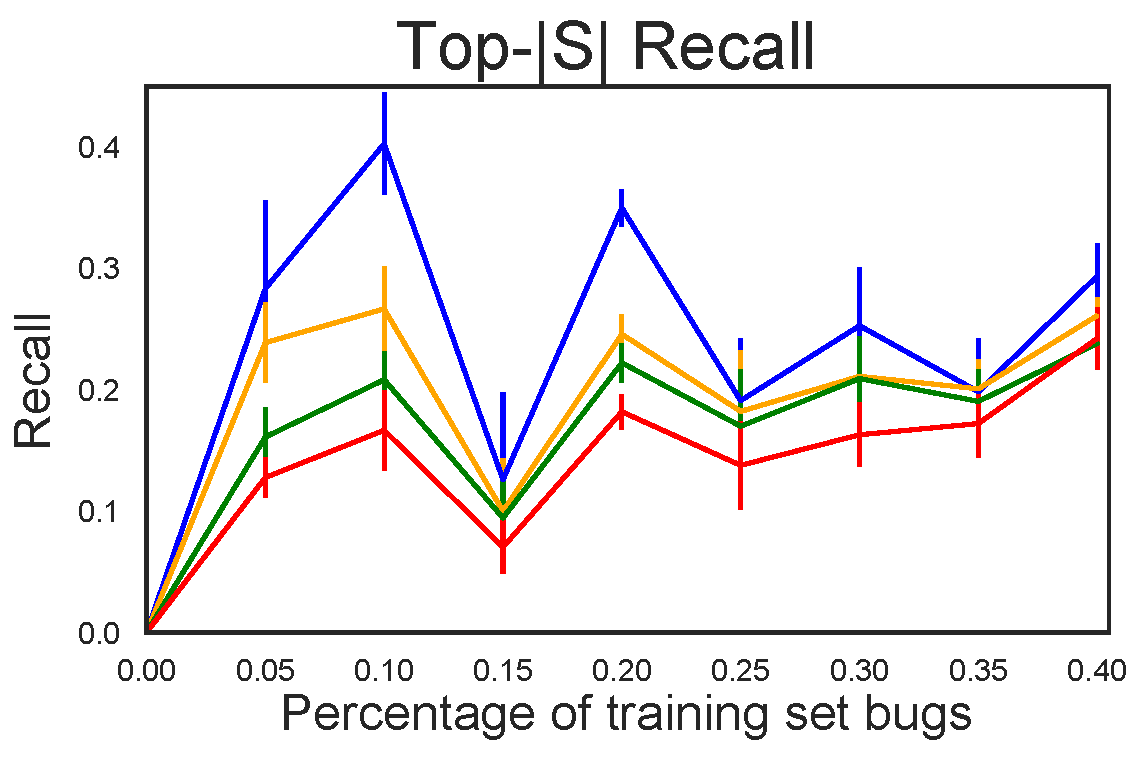}
    \\
    {\small{(b)}}
    \end{tabular} & 
    \begin{tabular}{c}
    \includegraphics[width=.33\textwidth]{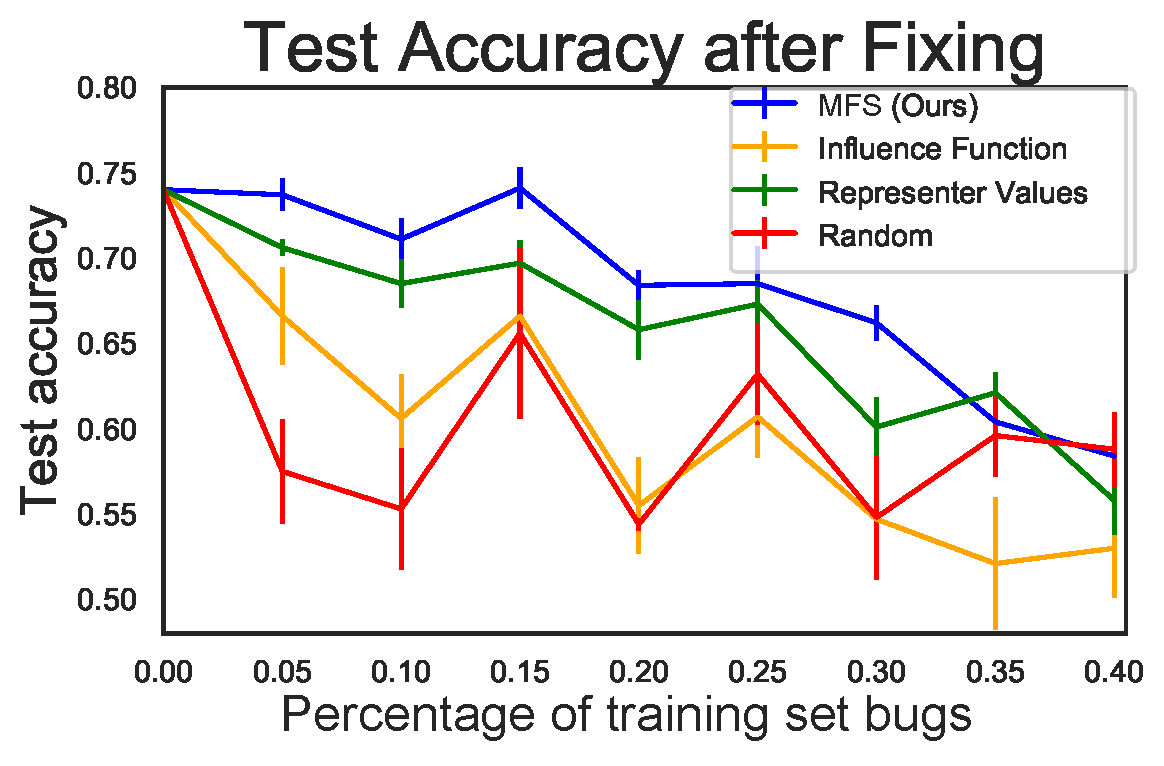} 
    \\
    {\small{(c)}}
\end{tabular}
\end{tabular}
\end{minipage}
    \caption{We conduct multiple experiments with different percentage of bugs in training set.
    (a) Precision and (b) Recall are measured to evaluate the correctness of debugging suggestions provided by each method. All algorithms have no access to the test data. MFS achieves the highest test accuracy (c) among others after respectively fixing their identified training bugs.}
    \label{fig:debug}
\end{figure*}




%% file: tex/4_experiment.tex
\section{Experiments}



We perform various quantitative and qualitative experiments to demonstrate the effectiveness of our approach in different application domains, including detection of data poisoning attack, training set debugging, explaining image classification, and understanding loan decisions. Since MFS is a sample-based explanation method, we use state-of-the-art works \cite{koh2017understanding, yeh2018representer} in the same category as our baselines.


\begin{figure*}[t!]
        \centering
        \includegraphics[width=.9\linewidth]{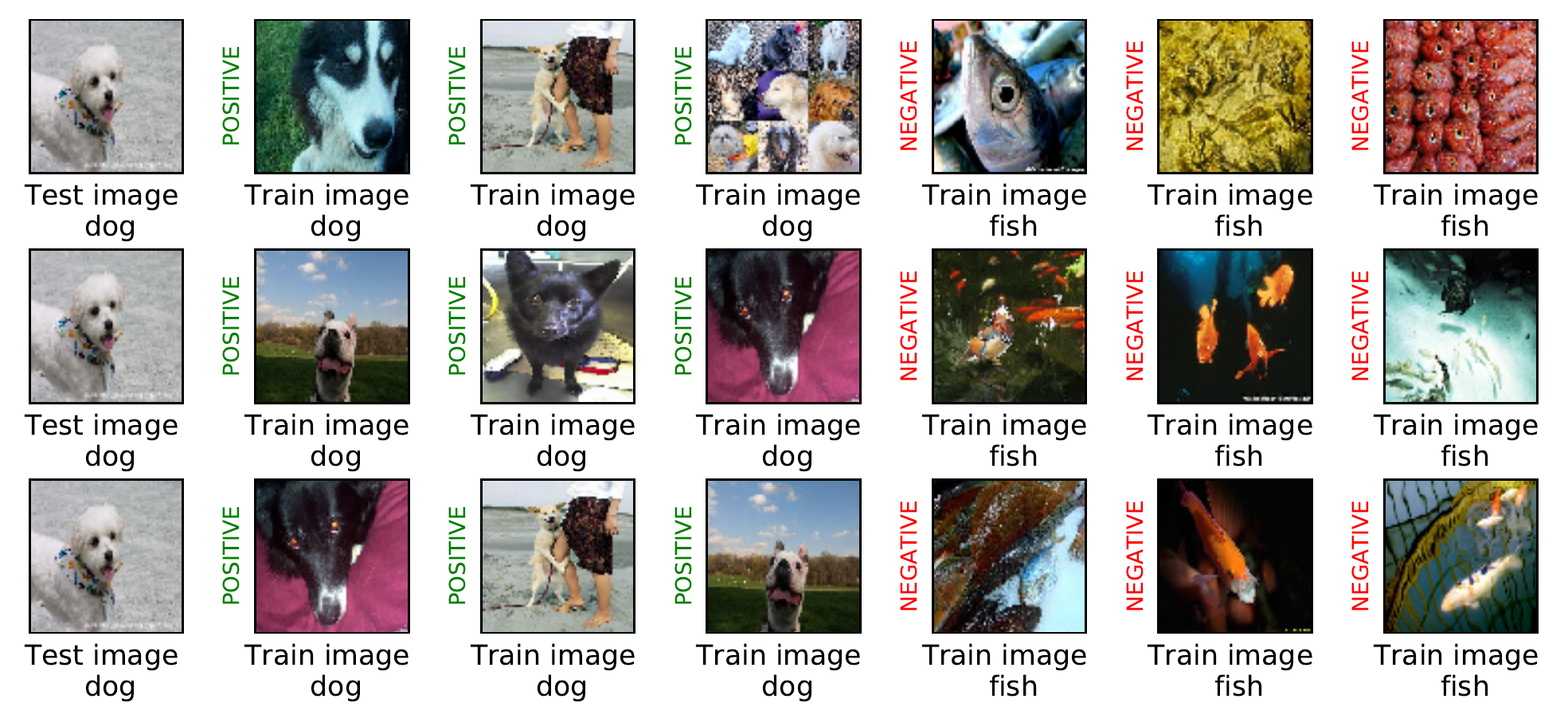}
        \caption{Comparison of the top three positive and negative influential training images for a test point (left-most column) using MFS (first row), influence functions (second row), and representer points method (third row).}
        \label{fig:showcase}
\end{figure*}

\subsection{Data Poisoning Detection}
Data poisoning is a type of attack on machine learning models in which an attacker manipulates the behavior of the model by adding adversarial examples to the training set. \cite{shafahi2018poison} has shown recently that complex neural network models, especially in transfer learning settings, are highly vulnerable to data poisoning attack: by crafting a poisoning sample that is close to the target testing data in the feature space but with a different label, the attacker can achieve a 100\% success rate of causing a misclassification. 
We can visually demonstrate the data poisoning process in Figure \ref{fig:poison}(a), where a poison instance with opposite class is placed close to the target data point and flips the decision boundary to cause misclassification. MFS then becomes a natural choice in detecting such poisoning samples within a training set by checking the most influential samples of a possible wrong decision.


We attacked a pretrained Inception V3 \cite{szegedy2016rethinking} network 
on ImageNet \cite{russakovsky2015imagenet} dog-vs-fish dataset. 
500 instances were selected from each class as the training data, and we randomly select 5 targets (T1 to T5) from the test set to create poisoning samples.
Figure \ref{fig:poison}(b) shows the size of MFS constructed by our method, before and after attacking. It requires a greater amount of supporting points to flip the decision in 
the clean dataset, than in the dataset with poisoning sample evaluated on the same test point.
This is expected, since the decision of the attacked test point will be flipped once the poisoning sample is found and removed. 
 Figure \ref{fig:poison_imgs} shows one instance of the data poisoning detection. 
 We can see MFS is able to correctly identify the harmful image as the top 1 support after the attack. Note that other sample-based explanation baselines, such as influence function \cite{koh2017understanding} and representer points \cite{yeh2018representer}, cannot be applied for this application. This is because they cannot measure the threshold for decision flipping, making it hard to identify the set $\mathcal{S}$ that supports a decision.


\subsection{Training Set Debugging}
Bugs and flaws in the training set can adversely affect learning models that generalized well. Because the whole training set is often too large for human inspection, it is highly desirable to have algorithms to automatically identify and prioritize the group of most problematic points for checking. 
Previous works \cite{koh2017understanding, yeh2018representer, khanna2018interpreting, zhang2018training} have suggested different approaches for solving this problem. 
However, 
%
these methods do not provide an indication of how much training data needs to be checked in order to improve the test accuracy to a desired level. 
%
%
Our method defines this threshold and achieves a better quantitative performance in generalization. 

We use the Enron1 spam dataset \cite{metsis2006spam} with 3317 training, 830 validation, and 1035 test examples, a certain percentage of the labels in the training set are randomly flipped. 
After training a logistic regression classifier on the noisy training set, 
we construct the MFS for incorrect predictions with high confidence in the validation set. The result shows there is a high probability that the bugs lie within MFS.
We compare our method with influence function \cite{koh2017understanding}, representer point \cite{yeh2018representer}, and random selection baselines. We analyze the top-$|\mathcal{S}|$ influential samples selected by each algorithm. 

Define $\mathcal{B}$ as the set that contains all the training set bugs, and $\mathcal{S}$ as the set of training samples selected by each method. Figure \ref{fig:debug} (a) and (b) demonstrate the precision ($\frac{|\mathcal{B} \cap \mathcal{S}|}{|\mathcal{S}|}$) and recall ($\frac{|\mathcal{B} \cap \mathcal{S}|}{|\mathcal{B}|}$) of the top-$|\mathcal{S}|$ selection. We find that MFS tends to provide the most correct suggestions, and consequently the highest test accuracy after fixing these bugs.   

\begin{figure*}[t!]
    \centering
    {\begin{tabu}{ccc}
        \includegraphics[width=.3\linewidth]{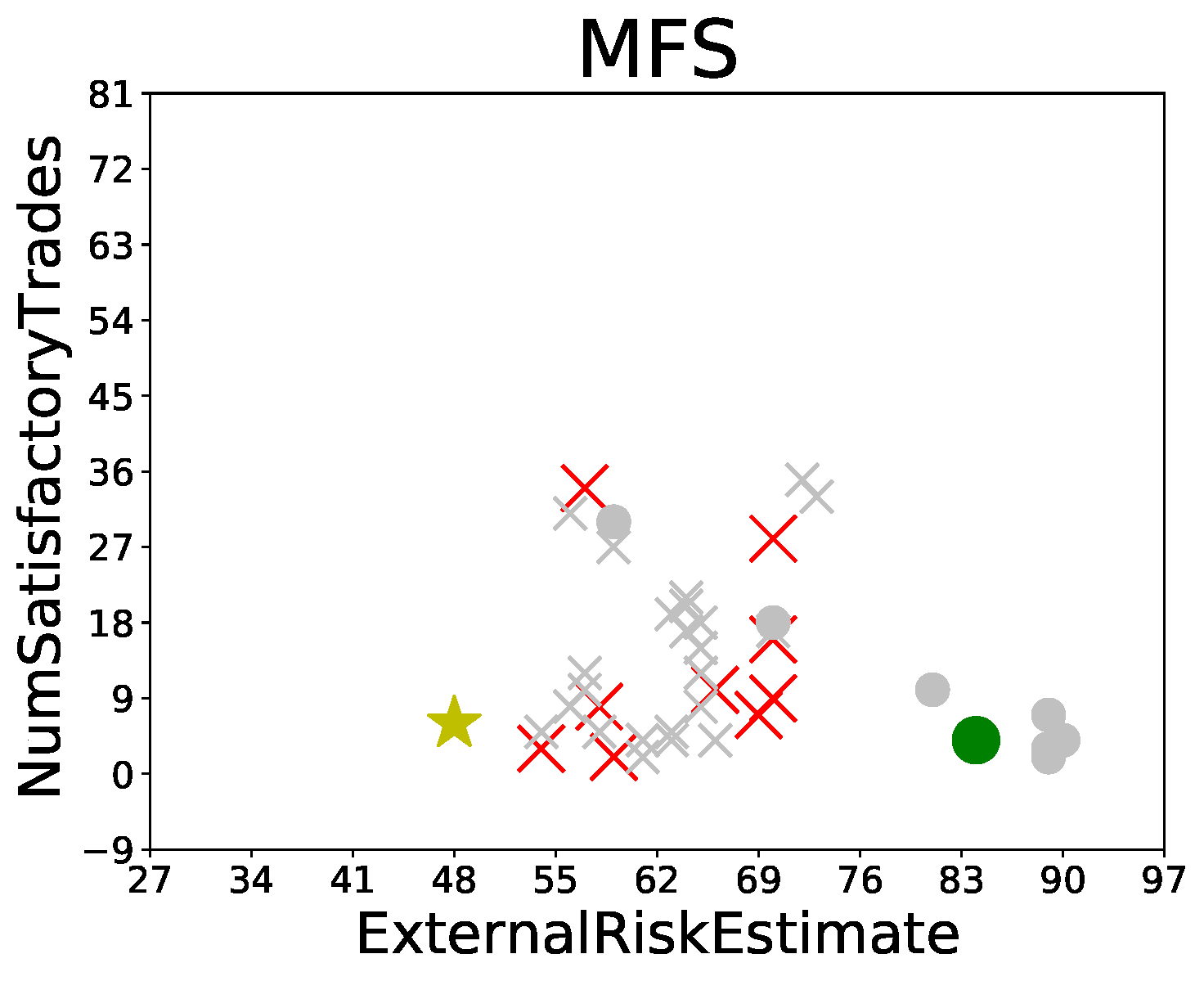} &  
        \includegraphics[width=.3\linewidth]{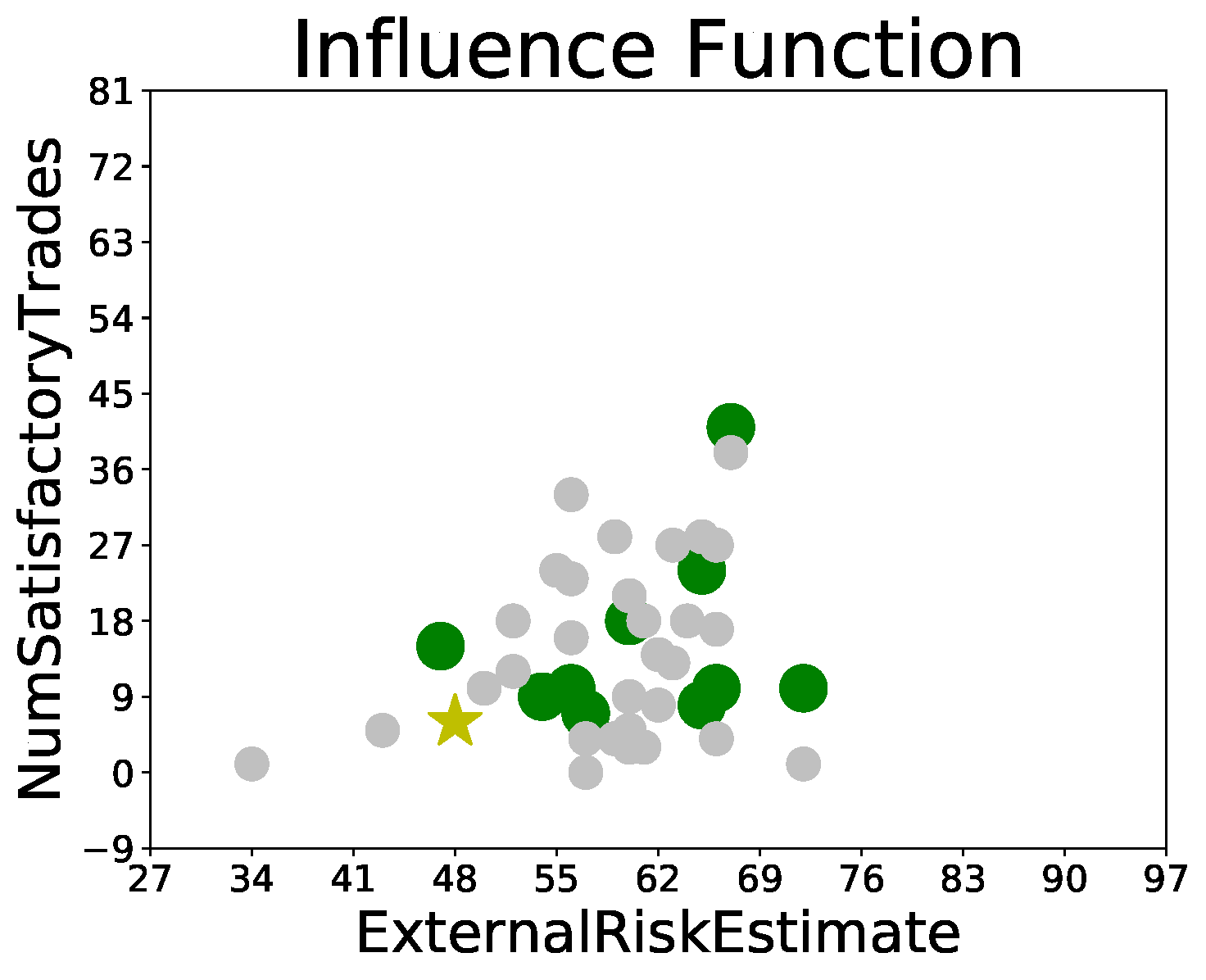} &
        \includegraphics[width=.326\linewidth]{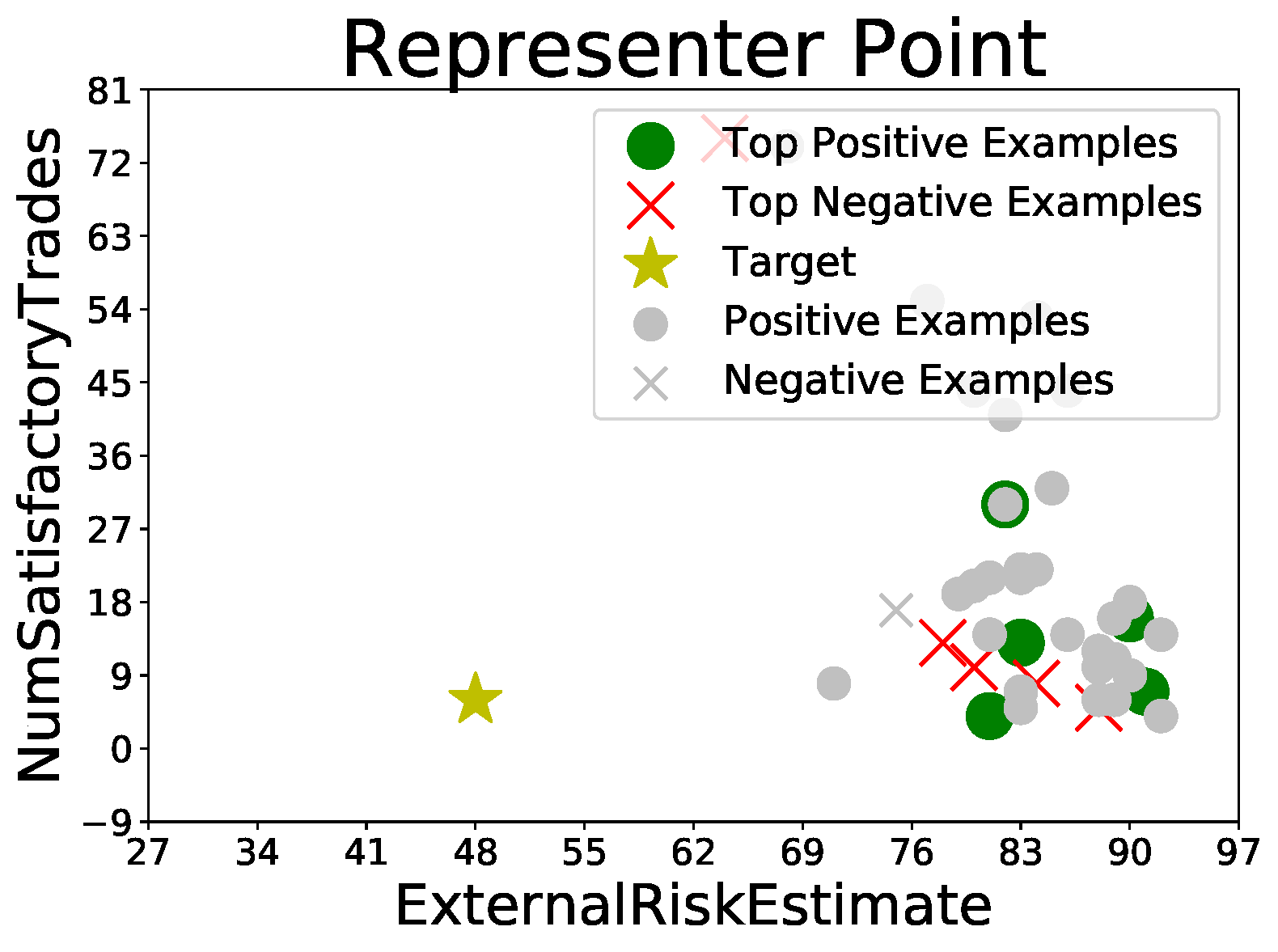} \\
        \includegraphics[width=.312\linewidth]{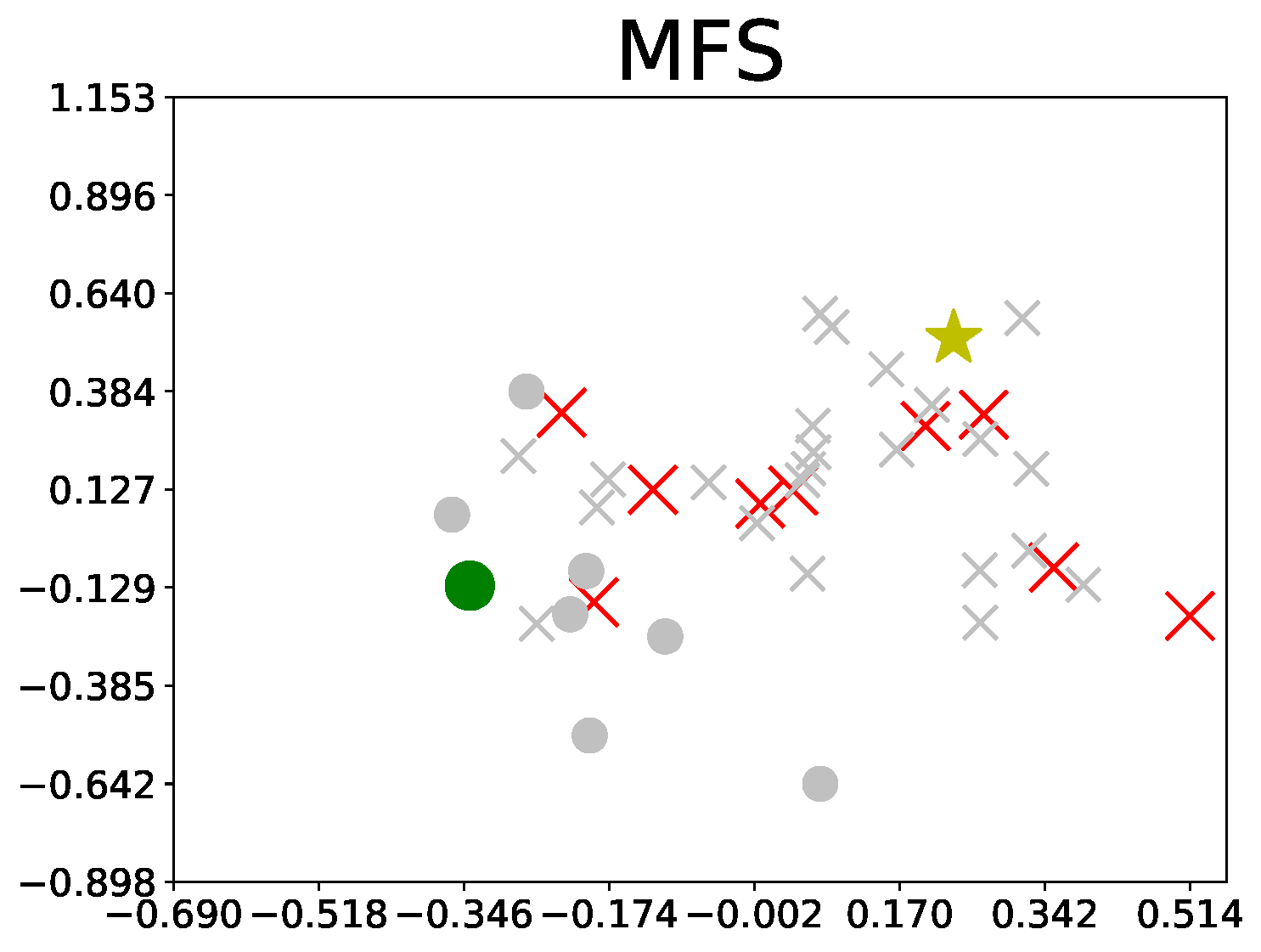} &  
        \includegraphics[width=.32\linewidth]{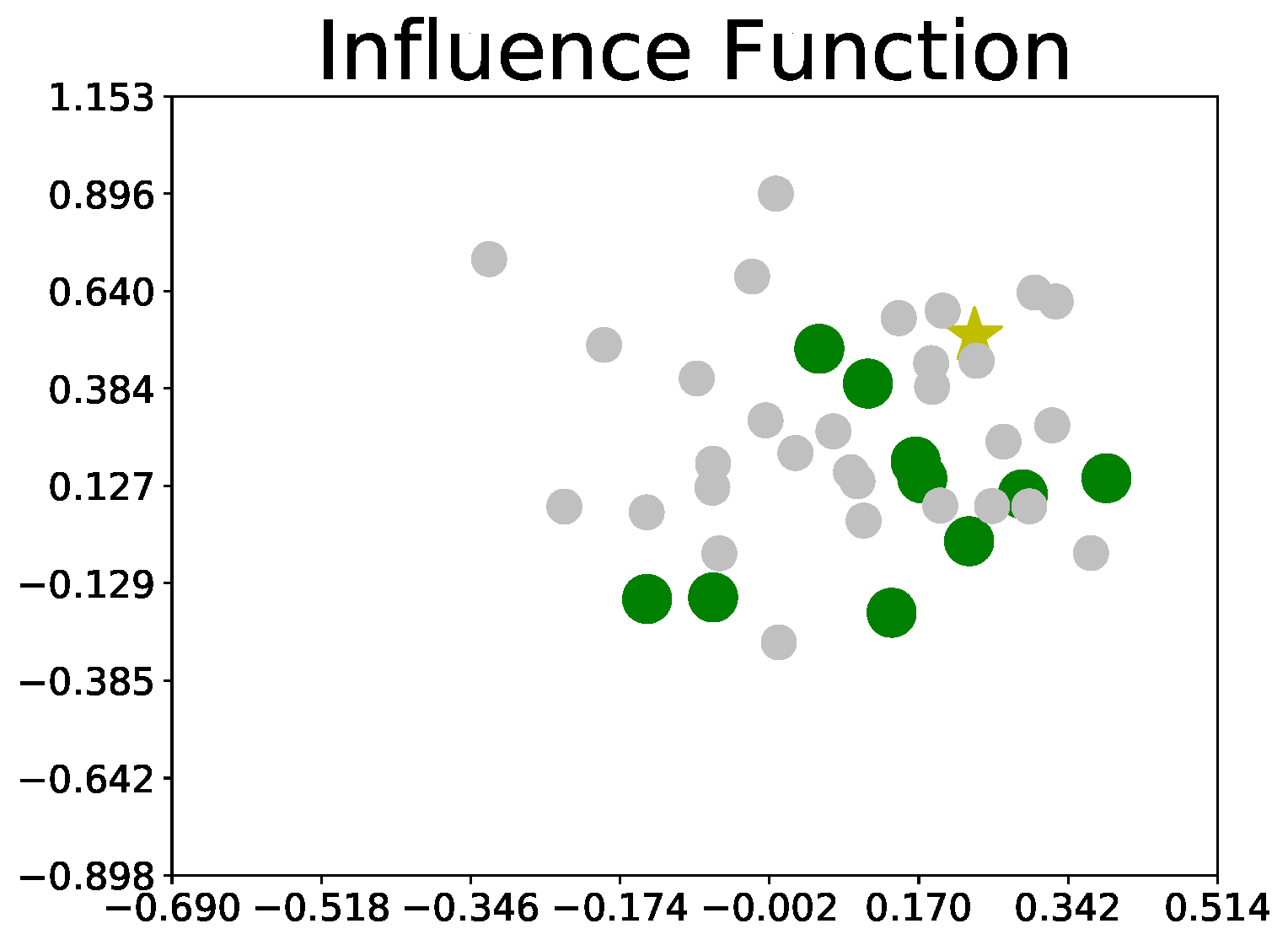} &
        \includegraphics[width=.317\linewidth]{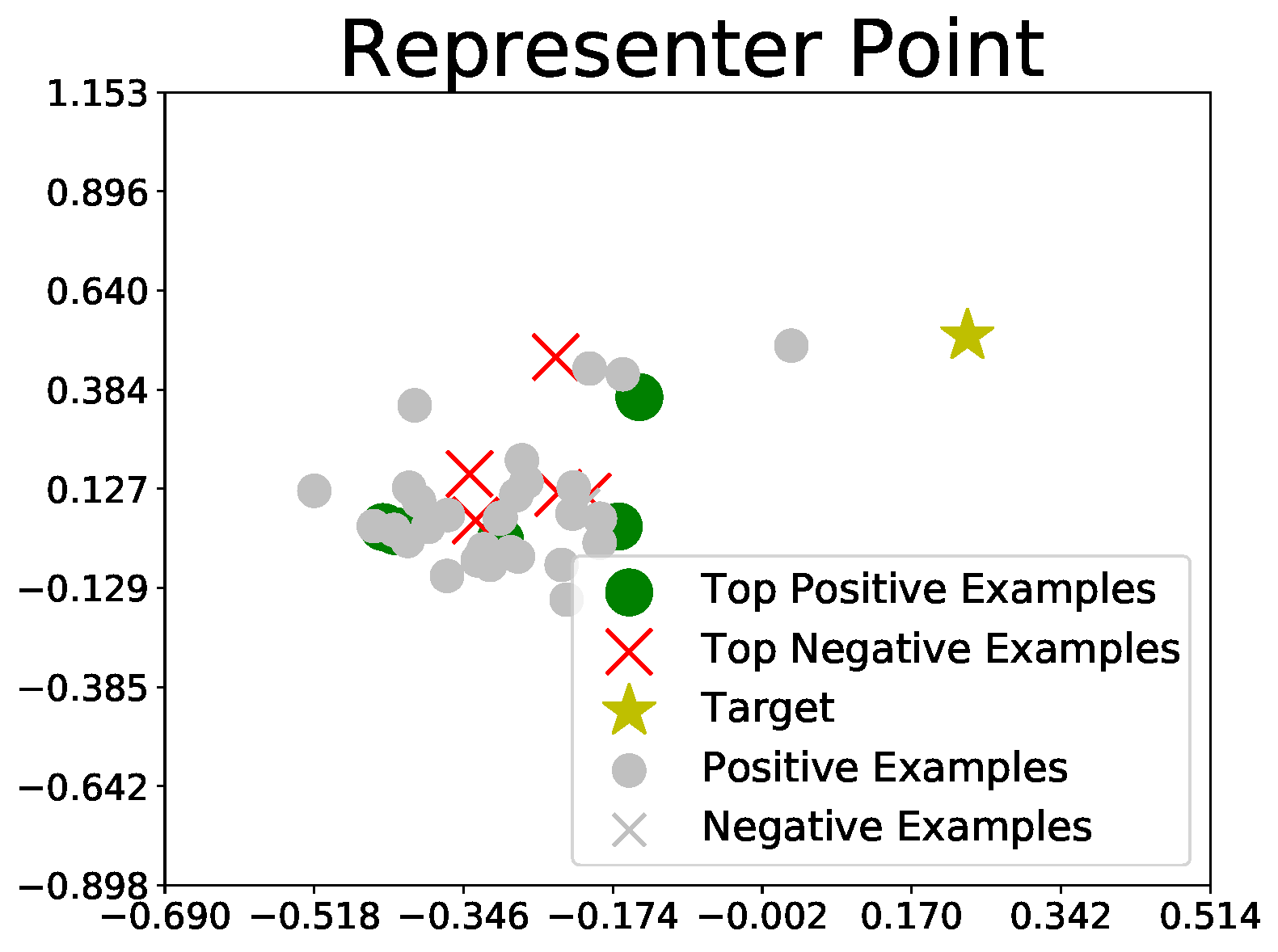} \\
        (a) & (b) & (c)\\
    \end{tabu}}
    \caption{Distribution of training samples that support the wrong prediction of target with 99\% confidence. First row: visualization in the space of two selected features; second row: visualization after PCA transformation. The top 10 positive / negative supports are labeled as green dot / red cross, the following 30 positive / negative supports are labeled as grey dot / grey cross.}
    \label{fig:finance}
\end{figure*}

\begin{table*}[h]
\caption{Time (in seconds) required for obtaining supporting dataset, influence function / representer value for all training points w.r.t. a selected target in test set. We run the experiments using one Nvidia RTX 2080-Ti GPU.}
\renewcommand{\arraystretch}{1.4}
    \centering
    \begin{tabular}{c||c||c||c}
    \hline
    \multicolumn{1}{c||}{Dataset} & \multicolumn{1}{c||}{MFS} & 
    \multicolumn{1}{c||}{Influence Function} & 
    \multicolumn{1}{c}{Representer Point}\\
    \hline
        ImageNet dog-fish & $0.75 \pm 0.06$ & $4.69 \pm 0.61$ & $0.17 \pm 0.05$ \\
        CIFAR-10 & $9.71 \pm 3.62$ & $337.59 \pm 79.22$ & $8.93 \pm 1.21$ \\
    \hline
    \end{tabular}
    \label{tab:my_label}
\end{table*}

\subsection{Most Relevant Positive and Negative Examples}
We visualize and compare our method with the influence function and representer point selection method on the ImageNet dog-fish classification dataset. 
%
We again use the pre-trained Inception V3 \cite{szegedy2016rethinking} model and fine-tune on the dog-fish classification dataset; the final test accuracy is 99\%. 
In our experiments, we generate a sequence of supporting points by the MFS and baselines, 
including both \emph{positive supporting points}, which have the same label as the target image, 
and \emph{negative supporting points}, which support the classifier's decision but are with a different label. 
%
Figure \ref{fig:showcase} demonstrates an example of the top three positive and negative points chosen by three approaches.  
The target image that we use is a white teddy shown in the left column. 
The top row shows the top three most positive and negative supporting images 
provided by MFS. 
The other two rows show the result of the influence function and representer points. 
We can see that MFS tends to choose more similar images from the same class, e.g.,  white and similar dogs, while the other two tend to give more different images, e.g., black dog. 
We find that some positive images are selected by multiple methods, demonstrating their significance to the prediction. 
%
On the other hand, we find that different methods tend to select different negative examples.  

\subsection{Understanding Loan Decisions}


We then evaluate our algorithm on the HELOC (Home Equity Line of Credit) credit application dataset, which is used in FICO 2018 xML challenge to explain lending decisions. Since the available winning approach is a global interpretable model, we only use this dataset to compare against our sample-based local explanation baselines.
Figure \ref{fig:finance} shows the explanations provided by MFS and baselines for a wrongly predicted target, 
where the data points are visualized on 
 two most highly weighted features (\texttt{ExternalRiskEstimate} and \texttt{NumSatisfactoryTrades}), as well as the entire dataset after PCA transformation.  
In this experiment, we choose a target in the test set with a positive label (good credit performance) but is predicted as negative (bad credit performance) with high confidence by the classifier.   
We can see from Figure~\ref{fig:finance} column (a) that the target is surrounded by samples with negative labels, with only a few positive examples at a further distance. This provides a reasonable explanation of the mistake and allows people to further inspect the negative supporting points to fix the problem. On the other hand, we find that the influence function method (Figure~\ref{fig:finance} column (b)) can only select positive supporting points with positive labels; 
although the selected points are close to the target, they do not explain how the mistake is made. In comparison, we find that the representer point method (Figure~\ref{fig:finance} column (c)) provides supporting points from both classes, but all with large Euclidean distance from the target and hence provide no convincing explanation to the mistake. 

\subsection{Advantages over Baseline Methods}
The above positive results demonstrate the superiority of MFS in providing decision explanations for classification tasks. One critical reason is MFS connects both the original decision and its counterfactuals with a small set of training samples, making the explanation more intuitive and efficient. Although methods like influence functions and representer points are technically sound, explaining model decisions based on slight perturbations on labels and model weights on the penultimate layer of NNs are not the most natural choice. 

Besides evaluating multiple real-world tasks, we keep track of the computational cost of constructing MFS on datasets, including ImageNet dog-fish classification (900 training images for each class) and the whole CIFAR-10. We find that the result compares favorably with influence function. 
We perform 10 experiments on different targets and report in Table~\ref{tab:my_label} the average and standard deviation of the computation time (including the fine-tuning when computing representer points). 
We can find that our method, although being a bit slower than the representer points method, 
achieves much lower computational cost than the influence function method. 
This is because we only need to obtain the supporting instances that contain a small number of training samples. Our strategy that performs efficient approximation for parameter update helps to scale on large datasets.  

%% file: tex/5_conclusion.tex
\section{Conclusion}




We developed a model-agnostic and intuitive algorithm to identify a set of responsible samples (MFS) that a given model decision relies on. We build the supporting set iteratively by solving a quadratic programming problem with linear constraint, and speed up the re-training process through updating the model parameter with one-step Newton. To demonstrate the effectiveness of our method empirically, we apply it to a variety of real-world tasks and multiple benchmark datasets, and compare the performance of MFS with other prevailing explanation approaches. Results show that MFS produces better quantitative results as well as more intuitive and reliable explanations for local model predictions. 

In future work, we would like to explore the application of our method in other domains where interpretability is desired and required, such as medical decisions, time-series/sequential data or language/speech models. We also plan to investigate the robustness of our method and look into the uniqueness of the selected set of training samples.